\documentclass[nohyperref]{article}

\usepackage{microtype}
\usepackage{graphicx}
\usepackage{booktabs} 

\usepackage{hyperref}

\usepackage[accepted]{icml2022}

\usepackage{amsmath}
\usepackage{amssymb}
\usepackage{mathtools}
\usepackage{amsthm}

\usepackage[capitalize,noabbrev]{cleveref}

\theoremstyle{plain}
\newtheorem{theorem}{Theorem}[section]
\newtheorem{proposition}[theorem]{Proposition}
\newtheorem{lemma}[theorem]{Lemma}
\newtheorem{corollary}[theorem]{Corollary}
\theoremstyle{definition}
\newtheorem{definition}[theorem]{Definition}
\newtheorem{assumption}[theorem]{Assumption}
\theoremstyle{remark}
\newtheorem{remark}[theorem]{Remark}

\usepackage[textsize=tiny]{todonotes}

\icmltitlerunning{Fat--Tailed VI with Anisotropic Tail Adaptive Flows}

\usepackage[bold]{hhtensor}
\usepackage{mathtools}
\usepackage{float}
\usepackage{braket}
\usepackage{dsfont}
\usepackage{graphbox}
\usepackage{pgf}
\usepackage{comment}
\usepackage{subcaption}
\usepackage{enumitem}
\DeclarePairedDelimiterX{\infdivx}[2]{(}{)}{%
  #1\;\delimsize\|\;#2%
}
\newcommand{\infdiv}{D\infdivx}

\newcommand{\dd}{\mathrm{d}}

\newcommand{\cE}{\mathcal{E}}
\newcommand{\cL}{\mathcal{L}}
\newcommand{\cO}{\mathcal{O}}

\newcommand{\cQ}{\mathcal{Q}}
\newcommand{\cS}{\mathcal{S}}

\newcommand{\simiid}{\overset{\text{iid}}{\sim}}

\newcommand{\mI}{\matr{I}}

\newcommand{\RR}{\mathbb{R}}
\newcommand{\EE}{\mathbb{E}}
\newcommand{\PP}{\mathbb{P}}

\newcommand{\cN}{\mathcal{N}}

\begin{document}

\twocolumn[
\icmltitle{Fat--Tailed Variational Inference with Anisotropic Tail Adaptive Flows}



\icmlsetsymbol{equal}{*}

\begin{icmlauthorlist}
\icmlauthor{Feynman Liang}{yyy,comp}
\icmlauthor{Liam Hodgkinson}{yyy,zzz}
\icmlauthor{Michael W. Mahoney}{yyy,zzz}
\end{icmlauthorlist}

\icmlaffiliation{yyy}{Department of Statistics, University of California, Berkeley, CA}
\icmlaffiliation{comp}{Meta, Menlo Park, CA}
\icmlaffiliation{zzz}{International Computer Science Institute, Berkeley, CA}

\icmlcorrespondingauthor{Feynman Liang}{feynman@berkeley.edu}

\icmlkeywords{Machine Learning, ICML}

\vskip 0.3in
]



\printAffiliationsAndNotice{}  

\begin{abstract}
    While fat-tailed densities commonly arise as posterior and marginal distributions in
    robust models and scale mixtures, they present challenges when Gaussian-based 
    variational inference fails to capture tail decay accurately. 
    We first improve previous theory on tails of Lipschitz flows 
    by quantifying how the tails affect the \emph{rate} of tail decay 
    and by expanding the theory to non-Lipschitz polynomial flows.
    Then, we develop an alternative theory for multivariate tail parameters which is 
    sensitive to tail-anisotropy. 
    In doing so, we unveil a fundamental problem which plagues many existing flow-based 
    methods: they can only model tail-isotropic distributions (i.e., distributions 
    having the same tail parameter in every direction).
    To mitigate this and enable modeling of tail-anisotropic targets, we propose 
    anisotropic tail-adaptive flows (ATAF).
    Experimental results on both synthetic and real-world targets confirm that ATAF 
    is competitive with prior work while also exhibiting appropriate tail-anisotropy.
\end{abstract}

\section{Introduction}
\label{sec:intro}


Flow based methods 
\citep{papamakarios2021normalizing}
have proven to be effective techniques to model complex
probability densities. They compete with the state of the art on
density estimation \citep{huang2018neural,durkan2019neural,jaini2020tails},
generative modeling \citep{chen2019residual,kingma2018glow}, and variational inference \citep{kingma2016improved,agrawal2020advances} tasks.
These methods start with a random variable $X$ having a simple and tractable
distribution $\mu$, and apply a learnable transport map $f_\theta$ to build
another random variable $Y = f_\theta(X)$ with a more expressive \emph{pushforward}
probability measure $(f_\theta)_\ast \mu$ \citep{papamakarios2021normalizing}.
In contrast to the implicit distributions \citep{huszar2017variational} produced by generative adversarial networks (GANs), flow based methods restrict the transport map $f_\theta$ to be invertible and to have efficiently-computable Jacobian determinants.
As a result, probability density functions can be tractably computed
through direct application of a change of variables
\begin{align}
    \label{eq:change-of-variable}
    p_{Y}(y)
      = p_{X}(f_\theta^{-1}(y)) \left\lvert \det
        \left.\frac{d f_\theta^{-1}(z)}{dz} \right\vert_{z=y}
      \right\rvert
\end{align}

While recent developments \citep{chen2019residual,huang2018neural,durkan2019neural} have focused primarily
on the transport map $f_\theta$, the base distribution $\mu$ has received comparatively less investigation. 
The most common choice for the base distribution is standard Gaussian $\mu = \cN(0,\mI)$.
However, in \Cref{thm:distn_class_closed}, we show this choice results in significant
restrictions to the expressivity of the model, limiting its utility for data that
exhibits fat-tailed (or heavy-tailed) structure.
Prior work addressing heavy-tailed flows \citep{jaini2020tails}
are limited to tail-isotropic base distributions ---
in \Cref{prop:isotropic-pushforward}, we also prove flows built on these base distributions
are unable to accurately model multivariate anisotropic fat-tailed structure.

\begin{figure*}[htbp]
  \centering
  \includegraphics{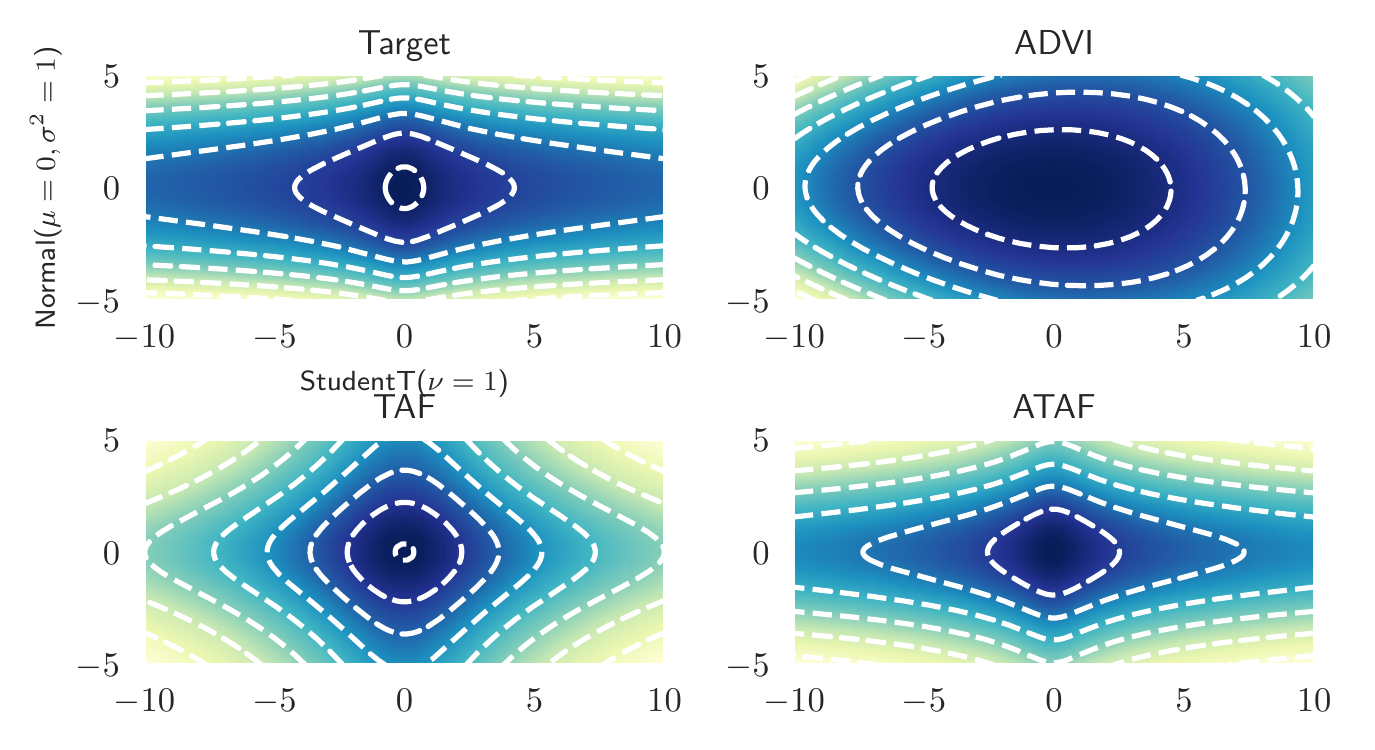}
    \vspace{-5mm}
  \caption{
    Variational inference against a tail-anisotropic target distribution $\cN(0,1) \otimes \text{StudentT}(\nu=1)$ (top left).
    Only ATAF (bottom right) is able to correctly reproduce the tail-anisotropy (fat-tailed along $x$-axis,
    Gaussian along $y$-axis).
    In contrast, ADVI's (top right) Gaussian base distribution and TAF's (bottom left) tail-isotropic $\prod_{i=1}^2 \text{StudentT}(\nu)$
    base distribution  can only model tail-isotropic distributions (\Cref{prop:isotropic-pushforward})
    which erroneously imposes power-law tails with the same rate of decay along both the $x$ and $y$ axes.
    \vspace{-5mm}
  }
  \label{fig:pancake}
\end{figure*}

Our work here aims to identify and address these deficiencies.
To understand the impact of the base distribution $\mu$ in flow-based models,
we develop and apply theory for fat-tailed random variables and their transformations under Lipschitz-continuous functions.
Our approach leverages the theory of concentration functions \citep[Chapter 1.2]{ledoux2001concentration} to significantly sharpen and extend prior results
\citep[Theorem 4]{jaini2019sum} by precisely describing the tail parameters of the pushforward distribution $(f_\theta)_\ast \mu$ under both Lipschitz-continuous (\Cref{thm:distn_class_closed}) and polynomial (\Cref{corr:closure_polynomials}) transport maps.
In the multivariate setting,
we develop a theory of direction-dependent tail parameters (\Cref{def:mv-tail-param}), and show that tail-isotropic base distributions yield tail-isotropic pushforward measures (\Cref{prop:isotropic-pushforward}). 
As a consequence of \Cref{prop:isotropic-pushforward}, prior methods \citep{jaini2020tails} are limited in that
they are unable to capture \emph{tail-anisotropy}.
This motivates the construction of \emph{anisotropic tail adaptive flows} (ATAF, \Cref{def:ataf}) as a means to
alleviate this issue (\Cref{remark:anisotropic}) and improve modeling of tail-anisotropic distributions.
Our experiments show ATAF exhibits correct tail behaviour in synthetic target distributions exhibiting fat-tails (\Cref{fig:cauchy_normal_student}) and tail-anisotropy (\Cref{fig:pancake}).
On realistic targets,
we find that ATAF can yield improvements in variational inference (VI) by capturing potential tail-anisotropy (\Cref{sec:experiments}).
\vspace{-3mm}
\subsection*{Related Work}
\vspace{-1mm}
\paragraph{Fat-tails in variational inference.}

Recent work in variational autoencoders (VAEs) have considered relaxing Gaussian assumptions to heavier-tailed distributions \citep{mathieu2019disentangling,chen2019residual,boenninghoff2020variational,abiri2020variational}.
In \citep{mathieu2019disentangling}, a StudentT prior distribution $p(z)$ is considered over the latent code
$z$ in a VAE with Gaussian encoder $q(z \mid x)$. They argue
that the anisotropy of a StudentT product distribution leads to more disentangled representations as compared to the standard choice of Normal distributions.
A similar modification is performed in \citep{chen2020use}, for
a coupled VAE (see \citep{cao2019coupled}). It showed improvements in the marginal
likelihoods of reconstructed images. In addition,
\citet{boenninghoff2020variational} consider a mixture of StudentTs for the
prior $p(z)$. To position
our work in context, note that the encoder $q(z \mid x)$ may be viewed
as a variational approximation to the posterior $p(z \mid x)$ defined by the
decoder model $p(x \mid z)$ and the prior $p(z)$. Our work differs from
\citep{mathieu2019disentangling,chen2020use,boenninghoff2020variational} in
that we consider fat-tailed variational approximations $q(z \mid x)$ rather
than priors $p(z)$. Although \citep{abiri2020variational} also considers
a StudentT approximate posterior, our work involves a more general
variational family which use normalizing flows.
Similarly, although \citep{wang2018variational} also deals with fat-tails in variational inference,
their goal is to improve $\alpha$-divergence VI by controlling the moments of importance
sampling ratios (which may be heavy-tailed). Our work here adopts
Kullback-Leibler divergence and is concerned with enriching the variational family
to include anisotropic fat-tailed distributions.
More directly comparable recent work \citep{ding2011t,futami2017expectation} studies the $t$-exponential family
variational approximation which includes StudentTs and other
heavier-tailed densities. Critically, the selection of their parameter $t$ (directly related to the
StudentT's degrees of freedom $\nu$), and the issue of tail anisotropy, are not discussed. 

\vspace{-1mm}
\paragraph{Flow based methods.}

Normalizing flows and other flow based methods have a rich history within variational
inference \citep{kingma2016improved,rezende2015variational,agrawal2020advances,webb2019improving}.
Consistent with our experience (\Cref{fig:blr-anisotropic}), \citep{webb2019improving}
documents normalizing flows can offer improvements over ADVI and NUTS across thirteen different
Bayesian linear regression models from \citep{gelman2006data}.
\citet{agrawal2020advances} shows that normalizing flows compose nicely with other
advances in black-box VI (e.g., stick the landing, importance weighting).
However, none of these works treat the issue of fat-tailed targets and inappropriate tail
decay.
To our knowledge, only TAFs \citep{jaini2020tails} explicitly consider flows with tails
heavier than Gaussians. Our work here can be viewed as a direct improvement of \citet{jaini2020tails},
and we make extensive comparison to this work throughout the body of this paper. At
a high level, we provide a theory for fat-tails which is sensitive to the rate of
tail decay and develop a framework to characterize and address the tail-isotropic limitations plaguing
TAFs.

\vspace{-2mm}
\section{Flow-based methods for fat-tailed variational inference}


\vspace{-2mm}
\subsection{Flow-based VI methods} 

The objective of VI is to approximate a target distribution $\pi(x)$ by searching over
a \emph{variational family} $\cQ = \{q_\phi : \phi \in \Phi\}$ of probability distributions $q_\phi$.
While alternatives exist \citep{li2016variational,wang2018variational}, VI typically
seeks to find $q_\phi$ ``close'' to $\pi$ as measured by Kullback-Leibler divergence $\infdiv{q_\phi}{\pi}$.
To ensure tractability without sacrificing generality, in practice \citep{wingate2013automated,ranganath2014black}
a Monte-Carlo approximation of the evidence lower bound (ELBO) is maximized:
\begin{align*}
  \text{ELBO}(\phi)
  &= \int q_\phi(x) \log \frac{\bar\pi(x)}{q_\phi(x)} dx\\
  & \approx \frac{1}{n} \sum_{i=1}^n \log \frac{\bar\pi(x_i)}{q_\phi(x_i)},\;
  x_i \simiid q_\phi,\;
  \bar{\pi} \propto \pi
\end{align*}
To summarize, this procedure enables tractable black-box VI
by replacing $\pi$ with $\bar\pi \propto \pi$ and approximating expectations with respect to $q_\phi$ (which are tractable only in simple variational families) through Monte-Carlo approximation. In Bayesian inference and probabilistic programming applications, the target posterior
$\pi(x) = p(x \mid y) = \frac{p(x, y)}{p(y)}$ is typically intractable but
$\bar\pi(x) = p(x,y)$ is computable (i.e. represented by the probabilistic program's
generative / forward execution).

\begin{table*}
  \centering
  \begin{tabular}{ccc}
    \toprule
    Model                                  & Autoregressive transform                                                                              & Suff. conditions for Lipschitz         \\
    \midrule
    NICE\citep{dinh2014nice}               & $z_j + \mu_j \cdot \mathds{1}_{k \not \in [j]}$                                                       & $\mu_j$ Lipschitz                      \\
    MAF\citep{papamakarios2017masked}      & $\sigma_j z_j + (1 - \sigma_j) \mu_j$                                                                 & $\sigma_j$ bounded                     \\
    IAF\citep{kingma2016improved}          & $z_j \cdot \exp(\lambda_j) + \mu_j$                                                                   & $\lambda_j$ bounded, $\mu_j$ Lipschitz \\
    Real-NVP\citep{dinh2016density}        & $\exp(\lambda_j \cdot \mathds{1}_{k \not\in[j]}) \cdot z_j + \mu_j \cdot \mathds{1}_{k \not \in [j]}$ & $\lambda_j$ bounded, $\mu_j$ Lipschitz \\
    Glow\citep{kingma2018glow}             & $\sigma_j \cdot z_j + \mu_j \cdot \mathds{1}_{k \not\in [j]}$                                         & $\sigma_j$ bounded, $\mu_j$ Lipschitz  \\
    NAF\citep{huang2018neural}             & $\sigma^{-1}(w^\top \cdot \sigma(\sigma_j z_j + \mu_j))$                                              & Always \par (logistic mixture CDF)          \\
    NSF\citep{durkan2019neural}            & $z_j \mathds{1}_{z_j \not\in [-B,B]} + M_j(z_j;z_{<j}) \mathds{1}_{x_j \in [-B,B]}$         & Always \par (linear outside $[-B,B]$)     \\
    FFJORD\citep{grathwohl2018ffjord} & n/a (not autoregressive)                                                                              & Always \par (required for invertibility)    \\
    ResFlow\citep{chen2019residual} & n/a (not autoregressive)                                                                              & Always \par (required for invertibility)    \\
    \bottomrule
  \end{tabular}
  \caption{Some popular / recently developed flows, the autoregressive transform used in the flow (if applicable),
  and sufficient conditions conditions for Lipschitz-continuity. A subset of this table was first presented
  in \citep{jaini2020tails}. $M(\cdot)$ denotes monotonic rational quadratic splines \citep{durkan2019neural}.
  \vspace{-5mm}
  }
  \label{tab:flows}
\end{table*}

While it is possible to construct a variational family $\cQ$ tailored to a specific task, we are interested in VI methods which are more broadly applicable and convenient to use: $\cQ$ should be automatically constructed from introspection of a given probabilistic model/program.
Automatic differentiation variational inference (ADVI) \citep{kucukelbir2017automatic} is an early implementation of automatic VI and it is still the default in certain probabilistic programming languages \citep{carpenter2017stan}.
ADVI uses a Gaussian base distribution $\mu$ and a transport map $f_\theta = f \circ \Phi_\text{Affine}$ comprised of an invertible affine transform composed with a deterministic transformation $f$ from $\RR$ to the target distribution's support (e.g., $\exp : \RR \to \RR_{\geq 0}$, $\text{sigmoid} : \RR \to [0,1]$).
As Gaussians are closed under affine transformations, ADVI's representational capacity is limited to deterministic transformations of Gaussians. Hence it cannot represent complex multi-modal distributions.
To address this, more recent work \citep{kingma2016improved,webb2019improving} replaces the affine map $\Phi_\text{Affine}$ with a flow $\Phi_{\text{Flow}}$ typically parameterized by an invertible neural network:
\begin{definition}
    \label[definition]{def:advi}
  ADVI (with normalizing flows) comprise the variational family
  $\cQ_\text{ADVI}~\coloneqq~\{
    (f~\circ~\Phi_\text{Flow})_\ast \mu
    \}$ 
  where $\mu = \text{Normal}(0_d, I_d)$,
  $\Phi_\text{Flow}$ is an invertible flow transform (e.g., \Cref{tab:flows})
  and $f$ is a deterministic bijection between constrained supports \citep{kucukelbir2017automatic}.
\end{definition}

As first noted in \citep{jaini2020tails}, the pushforward of a light-tailed Gaussian base distribution under a Lipschitz-continuous flow will remain light-tailed and provide poor approximation to fat-tailed targets. 
Despite this, 
many major probabilistic programming packages still make a default choice of Gaussian base distribution (\texttt{AutoNormalizingFlow}/\texttt{AutoIAFNormal} in Pyro \citep{bingham2019pyro}, \texttt{method=variational} in Stan \citep{carpenter2017stan}, \texttt{NormalizingFlowGroup} in PyMC \citep{patil2010pymc}).
To address this issue, tail-adaptive flows \citep{jaini2020tails} use a
base distribution $\mu_\nu = \prod_{i=1}^d \text{StudentT}(\nu)$
where a single degrees-of-freedom $\nu \in \RR$ is used across all $d$ dimensions. More precisely,
\begin{definition}
  \label[definition]{def:taf}
  Tail adaptive flows (TAF) comprise the variational family
  $\cQ_\text{TAF}
    \coloneqq \{
    (f \circ \Phi_\text{Flow})_\ast \mu_\nu
    \}$
  where $\mu_\nu = \prod_{i=1}^d \text{StudentT}(\nu)$ with $\nu$ shared across all $d$ dimensions,
  $\Phi_{\text{Flow}}$ is an invertible flow,
  and $f$ is a bijection between constrained supports \citep{kucukelbir2017automatic}.
  During training, the shared degrees of freedom $\nu$ is treated as an additional variational parameter.
\end{definition}

\vspace{-2mm}
\subsection{Fat-tailed variational inference}

\vspace{-1mm}
Fat-tailed variational inference (FTVI)
considers the setting where the target $\pi(x)$ is fat-tailed. 
Such distributions commonly arise during a standard
``robustification'' approach where light-tailed noise distributions are
replaced with fat-tailed ones \citep{tipping2005variational}. They also
appear when weakly informative prior distributions are used in Bayesian
hierarchical models \citep{gelman2006prior}.

To formalize these notions of fat-tailed versus light-tailed distributions, a quantitative classification for tails is required.
While prior work classified distribution tails according to quantiles and the existence of moment generating functions
\citep[Section 3]{jaini2020tails}, here we propose a more natural and finer-grained classification based upon
the theory of concentration functions \citep[Chapter 1.2]{ledoux2001concentration} which is sensitive to
the rate of tail decay.

\begin{definition}[Classification of tails]
    \label[definition]{def:tail-classification}
    For each $\alpha,p > 0$, we let 
    \vspace{-1mm}
    \begin{itemize}[leftmargin=*]
        \item $\mathcal{E}_\alpha^p$ denote the set of \emph{exponential-type} random variables $X$ with $\mathbb{P}(|X| \geq x) = \Theta(e^{-\alpha x^p})$;
    \vspace{-1mm}
        \item $\mathcal{L}_\alpha^p$ denote the set of \emph{logarithmic-type} random variables $X$ with $\mathbb{P}(|X| \geq x) = \Theta(e^{-\alpha(\log x)^p})$.
    \end{itemize}
    \vspace{-1mm}
    In both cases, we call $p$ the \emph{class index} and $\alpha$ the \emph{tail parameter} for $X$.
    Note that every $\cE_\alpha^p$ and $\cL_\beta^q$ are disjoint, that is,
    $\cE_\alpha^p \cap \cL_\beta^q = \emptyset$ for all $\alpha,\beta,p,q > 0$.
    For brevity, we define the ascending families
    $\overline{\cE_\alpha^p}$ and $\overline{\cL_\alpha^p}$
    analogously as before except with $\Theta(\cdot)$ replaced
    by $\cO(\cdot)$. Similarly, we denote the class of distributions with exponential-type
    tails with class index at least $p$
    by $\overline{\cE^p} = \cup_{\alpha \in \RR_+} \overline{\cE_\alpha^p}$, and similarly for $\overline{\cL^p}$.
\end{definition}

For example, $\overline{\cE_\alpha^2}$ corresponds to $\alpha^{-1/2}$-sub-Gaussian random variables,
    $\overline{\cE_\alpha^1}$ corresponds to sub-exponentials, and (of particular relevance to this paper) $\cL^1_\alpha$ corresponds to the class of power-law distributions.

\vspace{-2mm}
\section{Tail behavior of Lipschitz flows}

\vspace{-1mm}
This section states our main theoretical contributions; proofs are deferred to \Cref{sec:proofs}.
We sharpen previous impossibility results approximating fat-tailed targets
using light-tailed base distributions \citep[Theorem 4]{jaini2020tails}
by characterizing the effects of Lipschitz-continuous transport maps on not only the tail class
but also the class index and tail parameter (\Cref{def:tail-classification}). Furthermore, we extend the theory
to include polynomial flows \citep{jaini2019sum}. For the multivariate setting,
we define the tail-parameter function (\Cref{def:mv-tail-param}) to help formalize the notion
of tail-isotropic distributions and prove a fundamental limitation that tail-isotropic
pushforwards remain tail-isotropic (\Cref{prop:isotropic-pushforward}).

Most of our results are developed within the context of Lipschitz-continuous transport maps $f_\theta$.
In practice, many flow-based methods exhibit Lipschitz-continuity in their transport map either by design \citep{grathwohl2018ffjord,chen2019residual}, or as a consequence of choice of architecture and activation function (\Cref{tab:flows}). 
%
The following assumption encapsulates this premise.
\begin{assumption}\label[assumption]{assump:lipschitz}
    $f_\theta$ is invertible, and both $f_\theta$ and $f^{-1}_\theta$
    are $L$-Lipschitz continuous (e.g., sufficient conditions in \Cref{tab:flows} are satisfied).
\end{assumption}
It is worth noting that domains other than $\mathbb{R}^d$ may require an additional bijection between supports (e.g. $\exp : \mathbb{R} \to \mathbb{R}_+$) which
could violate \cref{assump:lipschitz}.

\vspace{-2mm}
\subsection{Closure of tail classes}
\label{ssec:failure}

\vspace{-1mm}
Our first set of results pertain to closure of the tail classes in \Cref{def:tail-classification}
under Lipschitz-continuous transport maps. While earlier work \citep{jaini2020tails} demonstrated
closure of exponential-type distributions $\cup_{p > 0} \overline{\cE^p}$ under flows satisfying \Cref{assump:lipschitz}, our results in Theorem \ref{thm:distn_class_closed}, and Corollaries \ref{corr:heavy_to_light} and \ref{corr:closure_polynomials} sharpen these observations, showing that (1) Lipschitz transport maps cannot decrease the class index $p$ for exponential-type random variables, but can alter the tail parameter $\alpha$; and
(2) under additional assumptions, cannot change either class index $p$ or the tail parameter $\alpha$ for logarithmic-type random variables.

\begin{theorem}[Lipschitz maps of tail classes]
  \label{thm:distn_class_closed}
  Under \Cref{assump:lipschitz},
  the distribution classes $\overline{\cE^p}$
  and $\overline{\cL^p_\alpha}$ (with $p,\alpha > 0$) are closed
  under every flow transformation in \Cref{tab:flows}.
\end{theorem}

Informally, \Cref{thm:distn_class_closed} asserts that light-tailed base distributions cannot be transformed
via Lipschitz transport maps into fat-tailed target distributions.
Note this does not violate universality theorems for certain flows \citep{huang2018neural}
as these results only apply in the infinite-dimensional limit. Indeed, certain exponential-type families (such as Gaussian mixtures) are dense in the class of \emph{all} distributions, including those that are fat-tailed.

Note that $\overline{\cL^p_\alpha} \supset \cE^q_\beta$ for all $p,q,\alpha,\beta$, so \Cref{thm:distn_class_closed}
by itself does not preclude transformations of fat-tailed base distributions to light-tailed targets.
Under additional assumptions on $f_\theta$, we further establish a partial converse that a fat-tailed base distribution's tail parameter is unaffected after pushfoward
hence heavy-to-light transformations are impossible. Note here there is no ascending union over
tail parameters (i.e., $\cL^p_\alpha$ instead of $\overline{\cL^p_\alpha}$).

\begin{corollary}[Closure of $\cL^p_\alpha$]
  \label[corollary]{corr:heavy_to_light}
  If in addition $f_\theta$ is smooth
  with no critical points on the interior or boundary of
  its domain, then $\cL_\alpha^p$ is closed.
\end{corollary}

This implies that simply fixing a fat-tailed base
distribution \emph{a priori} is insufficient; the tail-parameter(s) of the base distribution must be explicitly optimized alongside
the other variational parameters during training.
While these additional assumptions may seem restrictive, note that many flow transforms
explicitly enforce smoothness and monotonicity \citep{wehenkel2019unconstrained,huang2018neural,durkan2019neural}
and hence satisfy the premises. In fact, we can show a version of \Cref{thm:distn_class_closed} ensuring closure of exponential-type
distributions under polynomial transport maps which do not satisfy \Cref{assump:lipschitz}.
This is significant because it extends the closure results to
include polynomial flows such as sum-of-squares flows \citep{jaini2019sum}.

\begin{corollary}[Closure under polynomial maps]
    \label[corollary]{corr:closure_polynomials}
  For any $\alpha, \beta, p, q \in \RR_+$, there does not exist a
  finite-degree polynomial map from $\cE_\alpha^p$ into $\cL_\beta^q$.
\end{corollary}



\vspace{-2mm}
\subsection{Multivariate fat-tails and anisotropic tail adaptive flows}

\vspace{-1mm}
Next, we restrict attention to power-law tails $\cL^1_\alpha$ and develop a multivariate fat-tailed theory and notions of isotropic/anisotropic tail indices. Using our theory, we prove that both ADVI and TAF are fundamentally limited because they
are only capable of fitting tail-isotropic target measures (\Cref{prop:isotropic-pushforward}).
We consider anisotropic tail adaptive flows (ATAF): a density
modeling method which can represent tail-anisotropic distributions (\Cref{remark:anisotropic}).


For example, consider the target distribution shown earlier in \Cref{fig:pancake} formed as the product of $\mathcal{N}(0,1)$ and $\text{StudentT}(\nu=1)$ distributions.
The marginal/conditional distribution along a horizontal slice (e.g., the distribution of $\braket{X,e_0}$)
is fat-tailed, while along a vertical slice (e.g., $\braket{X,e_1}$) it is Gaussian.
Another extreme example of tail-anisotropy where the tail parameter for
$\braket{X,v}$ is different in every direction $v \in \cS^{1}$
is given in \Cref{fig:radial-fat-tail}. Here $\mathcal{S}^{d-1}$ denotes the $(d-1)$-sphere in $d$ dimensions. 
Noting that the tail parameter depends on the choice of direction, we are motivated to consider
the following direction-dependent definition of multivariate tail parameters. 

\begin{definition}
  \label[definition]{def:mv-tail-param}
  For a $d$-dimensional random vector $X$,
  its \emph{tail parameter function} $\alpha_X : \cS^{d-1} \to \bar{\RR}_+$
  is defined as
  $\alpha_X(v) = -\lim_{x \to \infty} \log \PP(\braket{v,X} \geq x) / \log x$ when the limit exists, and $\alpha_X(v) = +\infty$ otherwise.
  In other words, $\alpha_X(v)$ maps directions $v$ into the tail parameter of the corresponding one-dimensional projection $\braket{v,X}$. The random vector $X$ is \emph{tail-isotropic} if $\alpha_X(v) \equiv c$ is constant and
  \emph{tail-anisotropic} if $\alpha_X(v)$ is not constant but bounded.
\end{definition}

\begin{figure*}[htbp]
  \centering
  \includegraphics{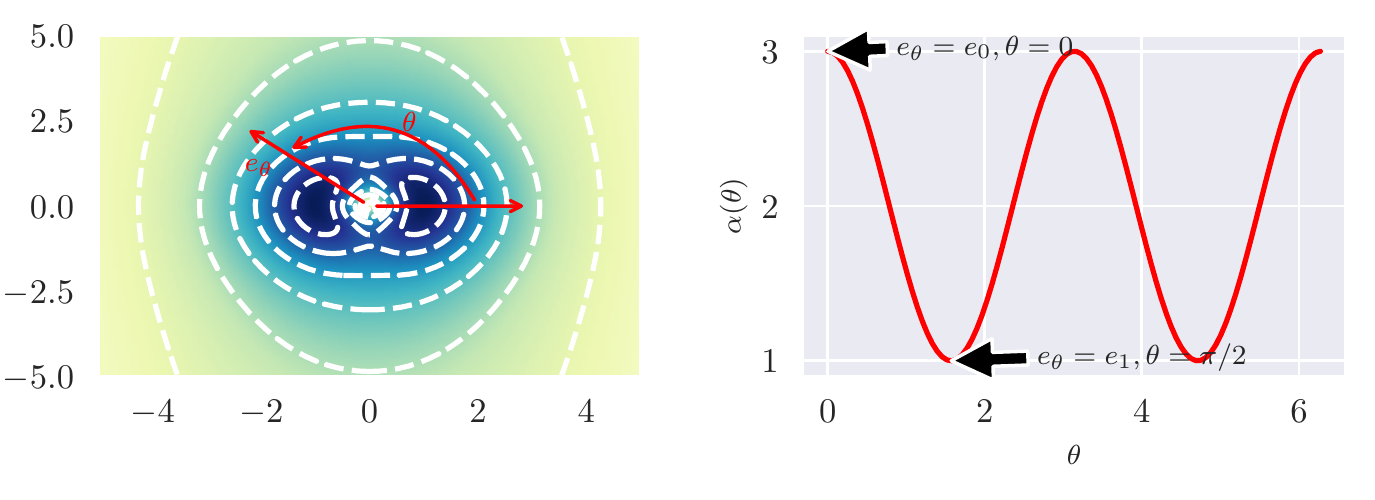}%
  \vspace{-7mm}
  \caption{
    Illustration of the direction-dependent tail-parameter function (bottom) on a tail-anisotropic distribution (top)
    with PDF $\dd P(r,\theta) = r^{-\alpha(\theta)} r \dd r \dd\theta$ and tail parameter $\alpha(\theta) = 2 + \cos(2\theta)$.
    While prior fat-tailed theory based on $\|X\|_2 = \sup_{\|v\|_2 = 1} \braket{X,v}$
    is only sensitive to the largest tail parameter $\max_{\theta \in [0, 2\pi]} \alpha(\theta) = 3.0$,
    our direction-dependent tail parameter function (bottom, red line)
    and its values along the standard basis axes ($\alpha(0)$ and $\alpha(\pi/2)$)
    capture \emph{tail-anisotropy}.
    \vspace{-5mm}
  }
  \label{fig:radial-fat-tail}
\end{figure*}
Of course, one can construct pathological densities where this definition is not effective
(see \Cref{eg:spiral}), but it will suffice for our purposes. 
It is illustrative to contrast with the theory presented for TAF \citep{jaini2020tails}
where only the tail exponent of $\|X\|_2$ is considered.
For $X = (X_1, \ldots, X_d)$ with $X_i \in \cL^1_{\alpha_i}$, 
by Fatou-Lebesgue and \Cref{lem:sum-rule}
\begin{align*}
  &\PP[\|X\|_2 \geq t]
  = \PP\left[\sup_{z \in \cS^{d-1}} \braket{X,z} \geq t\right]\\
  &\quad \geq \sup_{z \in \cS^{d-1}} \PP[\braket{X,z} \geq t]
  = \max_{1 \leq i \leq d} \nu_i
  = \max_{0 \leq i \leq d-1} \alpha_X(e_i)
\end{align*}
Therefore, considering only the tail exponent of $\|X\|_2$ is equivalent to summarizing $\alpha_X(\cdot)$ by an upper bound.
Given the absence of the tail parameters for other directions (i.e., $\alpha_X(v) \neq \sup_{\|v\|=1} \alpha_X(v)$)
in the theory for TAF \citep{jaini2020tails}, it should be unsurprising that both their multivariate 
theory as well as their experiments only consider tail-isotropic distributions obtained either
as an elliptically-contoured distribution with fat-tailed radial distribution or 
$\prod_{i=1}^d \text{StudentT}(\nu)$ (tail-isotropic by \Cref{lem:sum-rule}). 
Our next proposition shows that this presents a significant limitation when the target distribution is
tail-anisotropic.

\begin{proposition}[Pushforwards of tail-isotropic distributions]
  \label[proposition]{prop:isotropic-pushforward}
  Let $\mu$ be tail isotropic with non-integer parameter $\nu$
  and suppose $f_\theta$ satisfies \Cref{assump:lipschitz}.
  Then $(f_\theta)_\ast \mu$ is tail isotropic with parameter $\nu$.
\end{proposition}

To work around this limitation without relaxing \Cref{assump:lipschitz}, it is evident
that tail-anisotropic base distributions $\mu$ must be considered. Perhaps the most straightforward modification to incorporate a tail-anisotropic base distribution replaces TAF's isotropic base distribution $\prod_{i=1}^d \text{StudentT}(\nu)$
with $\prod_{i=1}^d \text{StudentT}(\nu_i)$. Note that $\nu$ is no longer shared across dimensions,
enabling $d$ different tail parameters to be represented:

\begin{definition}\label[definition]{def:ataf}
  Anisotropic Tail-Adaptive Flows (ATAF) comprise the variational family
  $\cQ_\text{ATAF}~\coloneqq~\{
    (f \circ \Phi_\text{Flow})_\ast \mu_\nu
    \},$
  where $\mu_\nu = \prod_{i=1}^d \text{StudentT}(\nu_i)$, each $\nu_i$ is \emph{distinct}, and $f$ is a bijection between constrained supports \citep{kucukelbir2017automatic}.
  Analogous to \cite{jaini2020tails}, ATAF's implementation
  treats $\nu_i$ identically to the other parameters in the flow and jointly optimizes over them.
\end{definition}

\begin{remark}\label[remark]{remark:anisotropic}
  Anisotropic tail-adaptive flows can represent tail-anisotropic distributions with up to $d$ different
  tail parameters while simultaneously satisfying \Cref{assump:lipschitz}.
  For example, if $\Phi_\text{Flow} = \text{Identity}$ and $\mu_\nu = \prod_{i=1}^d \text{StudentT}(i)$
  then the pushforward $(\Phi_\text{Flow})_\ast \mu_\nu = \mu_\nu$ is tail-anisotropic.
\end{remark}

Naturally, there are other parameterizations of the tail parameters $\nu_i$ that may be more effective depending on the application. For example, in high dimensions, one might prefer not to allow for $d$ unique indices, but perhaps only fewer. On the other hand, by using only $d$ tail parameters, an approximation error will necessarily be incurred when more than $d$ different tail parameters are present. \Cref{fig:radial-fat-tail} presents a worst-case scenario where the target distribution has a continuum of tail parameters. In theory, this density could itself be used as an underlying base distribution, although we have not found this to be a good option in practice. The key takeaway is that to capture several different tails in the target density, one must consider a base distribution that incorporates sufficiently many \emph{distinct} tail parameters. 

Concerning the choice of StudentT families, we remark that since $\text{StudentT}(\nu) \Rightarrow \cN(0,1)$ 
as $\nu \to \infty$, ATAF should still provide reasonably good approximations to target 
distributions in $\overline{\cE^2}$ by taking $\nu$ sufficiently large. This can be seen in practice in \Cref{sec:normal-normal-location-mixture}.

\vspace{-2mm}
\section{Experiments}
\label{sec:experiments}

\vspace{-1mm}
Here we validate ATAF's ability to improve
a range of probabilistic modeling tasks.
Prior work \citep{jaini2020tails} demonstrated improved
density modelling when fat tails are considered, and
our experiments are complementary by evaluating TAFs and ATAFs for variational inference tasks as well as demonstrating the effect of tail-anisotropy for modelling real-world financial returns and insurance claims datasets.
We implement using the 
\texttt{REDACTED FOR REVIEW} probabilistic programming language[REDACTED]
and the
\texttt{REDACTED FOR REVIEW} library for normalizing flows[REDACTED],
and we have open-sourced code for reproducing experiments[REDACTED, reviewers please see supplementary materials for code]. 
Additional details for the experiments are detailed in \Cref{sec:additional-exp-details}.


\begin{figure*}[htbp]
  \centering
  \vspace{-0.2cm}
  \includegraphics{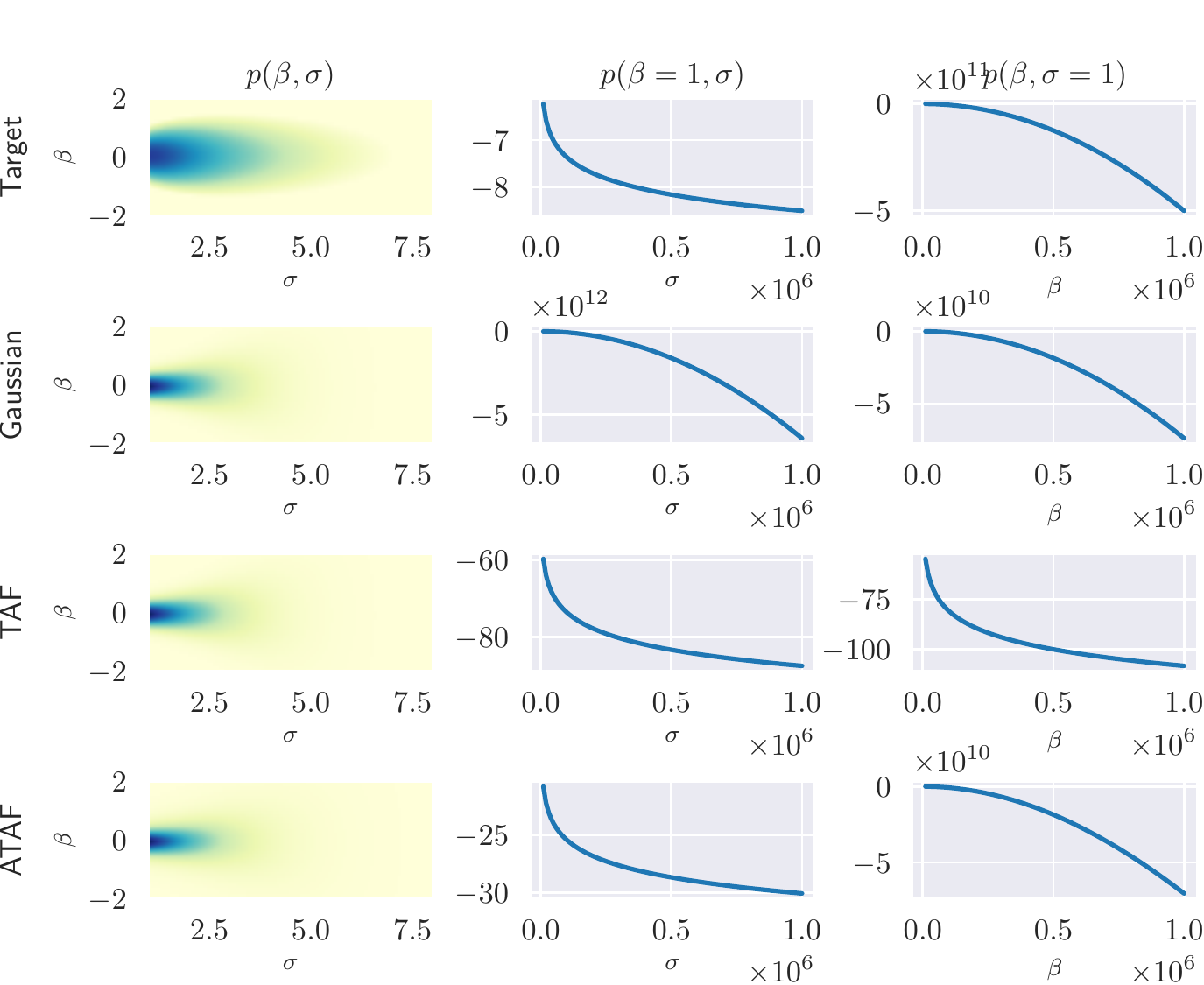}
  \vspace{-0.3cm}
  \caption{
    Bayesian linear regression's tail-anisotropic posterior
    (top left) exhibits a fat-tailed conditional in $\sigma$ (as evidenced by
    the convex power-law decay in the top middle panel) and a Gaussian conditional in $\beta$ (concave graph in top right panel).
    While all methods appear to provide a good approximation of the bulk (left column),
    \Cref{prop:isotropic-pushforward} implies
    Gaussian (Gaussian, second row) or isotropic StudentT product (TAF, third row) base distribution
    yields Gaussian or power-law tails respectively for \emph{both} $\sigma$ and $\beta$.
    In contrast, ATAF (bottom row) illustrates \Cref{remark:anisotropic} by simultaneously
    modeling a power-law tail on $\sigma$ and Gaussian tail on $\beta$.
  }
  \label{fig:blr-anisotropic}
\end{figure*}

\begin{table*}[htbp]
  \centering
    \begin{subfigure}[t]{0.49\textwidth}
        \centering
        \begin{tabular}{rcc}
            \toprule
                      & ELBO                & $\log p(y)$       \\
            \midrule
            ADVI      & $\mathbf{2873.90} \pm 6.95$    & $2969.73 \pm 1.73$ \\
            TAF       & $2839.64 \pm 9.10$    & $2973.85 \pm 0.87$ \\
            ATAF      & $2842.75 \pm 8.83$    & $\mathbf{2976.75} \pm 0.66$ \\
            \hline
            NUTS      & n/a                  & $3724.59 \pm 0.036$ \\
            \bottomrule
        \end{tabular}
        \caption{diamonds}
        \label{tab:diamonds}
    \end{subfigure}
    \begin{subfigure}[t]{0.49\textwidth}
        \centering
        \begin{tabular}{rcc}
            \toprule
                      & ELBO                & $\log p(y)$       \\
            \midrule
            ADVI      & $-72.13 \pm 6.89$    & $-53.25 \pm 3.44$ \\
            TAF       & $-64.64 \pm 4.88$    & $-52.51 \pm 4.41$ \\
            ATAF      & $\mathbf{-58.63} \pm 4.75$    & $\mathbf{-51.01} \pm 3.71$ \\
            \hline
            NUTS      & n/a                  & $-47.78 \pm 0.093$ \\
            \bottomrule
        \end{tabular}
        \caption{Eight schools}
        \label{tab:eight_schools}
    \end{subfigure}
    \vspace{-2mm}
        \caption{Monte-Carlo ELBO and importance weighted Monte-Carlo marginal likelihood
        $p(y) = \EE_{x \sim q_\theta} \frac{p(x,y)}{q_\theta(x)}$ (higher is better, $\pm$ standard errors)
        estimates from VI on real-world datasets.
        To understand the variational approximation gap, we include marginal likelihoods based on ``golden samples'' from \texttt{posteriordb} \citep{ghposteriordb} computed using No-U-Turn-Sampling (NUTS) \citep{hoffman2014no,carpenter2017stan}.
    }
  \label{fig:eight_schools}
  \vspace{-4mm}
\end{table*}

\begin{table*}[htbp]
    \centering
    \begin{tabular}{rcc}
        \toprule
                  & Fama-French 5 Industry Daily & CMS 2008-2010 DE-SynPUF       \\
        \midrule
        ADVI      & $-5.018 \pm 0.056$    & $-1.883 \pm 0.012$ \\
        TAF       & $-4.703 \pm 0.023$    & $-1.659 \pm 0.004$ \\
        ATAF      & $\mathbf{-4.699} \pm 0.024$    & $\mathbf{-1.603} \pm 0.034$ \\
        \bottomrule
    \end{tabular}
    \vspace{-2mm}
    \caption{
        Log-likelihoods (higher is better, $\pm$ standard errors) achieved on
        density modeling tasks involving financial returns \citep{fama2015five} and insurance claims \citep{cms} data. 
        \vspace{-5mm}
    }
    \label{tab:density-estimation}
\end{table*}

\subsection{Bayesian linear regression}

Consider one-dimensional Bayesian linear regression (BLR)
with conjugate priors, defined by priors and likelihood
\begin{gather*}
    \sigma^2 \sim \text{Inv-Gamma}(a_0, b_0)\\
    \beta \mid \sigma^2 \sim \cN(0, \sigma^2),\qquad
    y \mid X, \beta, \sigma \sim \cN(X \beta, \sigma^2) 
\end{gather*}
where $a_0$, $b_0$ are hyperparameters and the task is to approximate the posterior
distribution $p(\beta,\sigma^2 \mid X, y)$. Owing to conjugacy,
the posterior distribution can be explicitly computed. Indeed, $p(\beta,\sigma^2 \mid X, y) = \rho(\sigma^2)\rho(\beta \mid \sigma)$ where $\rho(\beta \mid \sigma) = \cN(\Sigma_n(X^\top X \hat\beta), \sigma^2 \Sigma_n)$, $\Sigma_n = (X^\top X + \sigma^{-2})^{-1}$, $\hat\beta = (X^\top X)^{-1} X^\top y$, and
    \[
    \rho(\sigma^2) = \text{Inv-Gamma}\bigg(
    a_0 + \frac{n}{2}, 
    b_0 + \frac{1}{2}(y^\top y - \mu_n^\top \Sigma_n \mu_n)\bigg)
    \]
This calculation reveals that the posterior distribution is tail-anisotropic:
for fixed $c$ we have that $p(\sigma^2, \beta=c \mid X, y) \propto \rho(\sigma^2) \in \cL^1_{\alpha_n}$
as a function of $\sigma$ (with $\alpha_n$ a function of $n$)
and $p(\sigma^2=c, \beta \mid X, y) \propto \rho(\beta \mid c) \in \overline{\cE^2}$
as a function of $\beta$.
As a result of \Cref{prop:isotropic-pushforward}, we expect ADVI and TAF to erroneously impose
Gaussian and power-law tails respectively for both $\beta$ and $\sigma^2$ as neither method
can produce a tail-anisotropic pushforward. This intuition is confirmed in \Cref{fig:blr-anisotropic},
where we see that only ATAF is the only method capable of modeling the tail-anisotropy present.

Conducting Bayesian linear regression is among the standard tasks requested of a probabilistic programming language, yet still displays tail-anisotropy. To accurately capture large quantiles, this tail-anisotropy  should not be ignored, necessitating a method such as ATAF.

\subsection{Diamond price prediction using non-conjugate Bayesian regression}
\label{ssec:diamonds}

Without conjugacy, the BLR posterior is intractable and there is no reason \emph{a priori} to expect tail-anisotropy.
Regardless, this presents a realistic and practical scenario for evaluating ATAF's ability to improve VI.
For this experiment, we consider BLR on the \texttt{diamonds} dataset \citep{wickham2011ggplot2} included in
\texttt{posteriordb} \citep{ghposteriordb}.
This dataset contains a covariate matrix $X \in \RR^{5000 \times 24}$ consisting of $5000$
diamonds each with $24$ features as well as an outcome variable $y \in \RR^{5000}$ representing each diamond's price.
The probabilistic model for this inference task is specified in Stan code provided by \citep{ghposteriordb} and is reproduced
for convenience
\begin{gather*}
    \alpha \sim \text{StudentT}(\nu=3, \text{loc}=8, \text{scale}=10)\\
    \sigma \sim \text{HalfStudentT}(\nu=3, \text{loc}=0, \text{scale}=10)\\
    \beta \sim \cN(0, \mI_{24}),\qquad
    y \sim \cN(\alpha + X \beta, \sigma)
\end{gather*}

For each VI method, we performed 100 trials each consisting of 5000 descent steps
on the Monte-Carlo ELBO estimated using 1000 samples and report the results in
\Cref{tab:diamonds}. We report both the final Monte-Carlo ELBO
as well as a Monte-Carlo importance-weighted approximation to
the log marginal likelihood $\log p(y) = \log \EE_{x \sim q_\theta} \frac{p(x,y)}{q_\theta(y)}$
both estimated using 1000 samples.







\subsection{Eight schools SAT score modelling with fat-tailed scale mixtures}
\label{ssec:eight_schools}

The eight-schools model \citep{rubin1981estimation,gelman2013bayesian} is a classical
Bayesian hierarchical model used originally to consider the relationship between standardized
test scores and coaching programs in place at eight schools.
A variation utilizing half Cauchy non-informative priors \citep{gelman2006prior} provides
a real-world inference problem involving fat-tailed distributions, and is formally specified
by the probabilistic model
\begin{gather*}
    \tau \sim \text{HalfCauchy}(\text{loc}=0, \text{scale}=5)\\
    \mu \sim \cN(0, 5),\qquad
    \theta \sim \cN(\mu, \tau),\qquad
    y \sim \cN(\theta, \sigma)
\end{gather*}
Given test scores and standard errors $\{(y_i, \sigma_i)\}_{i=1}^8$, we are interested in the
posterior distribution over treatment effects $\theta_1,\ldots,\theta_d$. The experimental
parameters are identical to \Cref{ssec:diamonds} and results are reported in \Cref{tab:eight_schools}.

\subsection{Financial and actuarial applications}

To examine the advantage of tail-anisotropic modelling in practice, we considered two benchmark datasets from financial (daily log returns for five industry indices during 1926--2021, \cite{fama2015five}) and actuarial (per-patient inpatient and outpatient cumulative Medicare/Medicid (CMS) claims during 2008--2010, \cite{cms}) applications where practitioners actively seek to model fat-tails and account for black-swan events. Identical flow architectures and optimizers were used in both cases, with log-likelihoods presented in \Cref{tab:density-estimation}. Both datasets exhibited superior fits after allowing for heavier tails, with a further improved fit using ATAF for the CMS claims dataset. 








\section{Conclusion}
\label{sec:conclusion}

In this work, we have sharpened existing theory for approximating fat-tailed distributions with normalizing flows, and formalized tail-(an)isotropy through a direction-dependent tail parameter. With this, we have shown that many prior flow-based methods are inherently limited by tail-isotropy. With this in mind, we proposed a simple flow-based method capable of modeling tail-anisotropic targets.
As we have seen, anisotropic FTVI is already applicable in fairly elementary examples such as Bayesian linear regression
and ATAFs provide one of the first methods for using the representational capacity of flow-based methods
while simultaneously producing tail-anisotropic distributions. A number of open problems still remain, including the study of other parameterizations of the tail behaviour of the base distribution. Even so, going forward, it seems prudent that density estimators, especially those used in black-box settings, consider accounting for tail-anisotropy using a method such as ATAF.


\bibliography{refs_ftvi}
\bibliographystyle{icml2022}

\newpage
\appendix
\onecolumn

\section{Proofs}
\label{sec:proofs}

\begin{proof}[Proof of \Cref{thm:distn_class_closed}]
  \label{proof:distn_class_closed}
  Let $X$ be a random variable from either $\cE_\alpha^p$
  or $\cL_\alpha^p$.
  Its concentration function
  (Equation 1.6 \cite{ledoux2001concentration}
  is given by
  \[
    \alpha_X(r)
    \coloneqq \sup \{ \mu\{x : d(x,A) \geq r\}; A \subset \text{supp}~X, \mu(A) \geq 1/2\}
    = \PP(\lvert X - m_X \rvert \geq r)
  \]
  Under Assumption 1, $f_\theta$ is Lipschitz (say with Lipschitz
  constant $L$) so by Proposition 1.3 of \cite{ledoux2001concentration},
  \[
    \PP(\lvert f_\theta(X) - m_{f_\theta(X)}\rvert \geq r)
    \leq 2 \alpha_X(r/L)
    = \cO(\alpha_X(r/L)),
  \]
  where $m_{f_\theta(X)}$ is a median of $f_\theta(X)$.
  Furthermore, by the triangle inequality
  \begin{align}
    \PP(\lvert f_\theta(X) \rvert \geq r)
    &= \PP(\lvert f_\theta(X) - m_{f_\theta(X)} + m_{f_\theta(X)} \rvert \geq r) \nonumber\\
    &\leq \PP(\lvert f_\theta(X) - m_{f_\theta(X)}\rvert \geq r - \lvert m_{f_\theta(X)}\rvert ) \nonumber\\
    &= \cO(\PP(\lvert f_\theta(X) - m_{f_\theta(X)}\rvert \geq r)) \nonumber\\
    &= \cO(\alpha_X(r/L)) \label{eq:pushforward-conc-fn}
  \end{align}
  where the asymptotic equivalence holds because $\lvert m_{f_\theta(X)} \rvert$ is independent of $r$.
  When $X \in \cE_\alpha^p$, \Cref{eq:pushforward-conc-fn} implies
  \[
    \PP(\lvert f_\theta(X) \rvert \geq r)
    = \cO(e^{-\frac{\alpha}{L} r^p}) \implies f_\theta(X) \in \overline{\cE}_{\alpha/L}^p,
  \]
  from whence we find that the Lipschitz transform of exponential-type
  tails continues to possess exponential-type tails with the same
  class index $p$, although the tail parameter may have changed. Hence,
  $\overline{\cE^p}$ is closed under Lipschitz maps for each $p \in \RR_{>0}$.
  On the other hand, when $X \in \cL_\alpha^p$, \Cref{eq:pushforward-conc-fn} also implies that
  \begin{align*}
    \PP(\lvert f_\theta(X) \rvert \geq r)
    &= \cO(e^{-\alpha (\log (r/L))^p})
    = \cO(e^{-\alpha (\log r)^p}),
  \end{align*}
  and therefore, $f_\theta(X) \in \overline{\cL_\alpha^p}$.
  Unlike exponential-type tails, Lipschitz transforms of
  logarithmic-type tails not only remain logarithmic, but
  their tails decay no slower than a logarithmic-type tail
  of the same class index with the \emph{same} tail parameter $\alpha$.
  This upper bound suffices to show closure under Lipschitz maps for the
  ascending family $\overline{\cL_\alpha^p}$.
\end{proof}

\begin{proof}[Proof of \Cref{corr:heavy_to_light}]
    Let $f_\theta$ be as before with the additional assumptions.
    Since $f_\theta$ is a smooth continuous bijection, it is a diffeomorphism.
    Furthermore, by assumption $f_\theta$ has invertible Jacobian on the closure of its
    domain hence $\sup_{x \in \text{dom}~f_\theta} \lvert (f_\theta)'(x) \rvert \geq M > 0$.
    By the inverse function theorem, $(f_\theta)^{-1}$ exists and is
    a diffeomorphism with
    \[
    \frac{d}{dx}(f_\theta)^{-1}(x) = \frac{1}{(f_\theta)'((f_\theta)^{-1}(x))} \leq \frac{1}{M}
    \]
    Therefore, $(f_\theta)^{-1}$ is $M^{-1}$-Lipschitz and we may apply
    \Cref{thm:distn_class_closed} to conclude the desired result.
\end{proof}

\begin{proof}[Proof of \Cref{corr:closure_polynomials}]
  Let $X \in \cE^p_\alpha$.
  By considering sufficiently large $X$ such that leading powers dominate, it suffices to consider monomials $Y = X^k$.
  Notice $\PP(Y \geq x) = \PP(X \geq x^{1/k}) = \Theta(e^{-\alpha x^{p/k}})$, and so
  $Y \in \cE^{p/k}_\alpha$. The result follows by disjointness of $\mathcal{E}$ and $\mathcal{L}$. 
\end{proof}

\begin{lemma}
    \label[lemma]{lem:sum-rule}
    Suppose $X \in \cL^1_\alpha$ and $Y \in \cL^1_\beta$.
    Then $X + Y \in \cL^1_{\min\{\alpha,\beta\}}$.
\end{lemma}

\begin{proof}
First, let $\gamma=\min\{\alpha,\beta\}$. It will suffice to show that (I) $\mathbb{P}(|X+Y|\geq r)=\mathcal{O}(r^{-\gamma})$, and (II) $\mathbb{P}(|X+Y|\geq r)\geq\Theta(r^{-\gamma})$. Since $(X,Y)\mapsto|X+Y|$ is a 1-Lipschitz function on $\mathbb{R}^{2}$ and $\mathbb{P}(|X|\geq r)+\mathbb{P}(|Y|\geq r)=\mathcal{O}(r^{-\gamma})$, (I) follows directly from the hypotheses and Proposition 1.11 of \cite{ledoux2001concentration}. To show (II), note that for any $M>0$, conditioning on the event $|Y|\leq M$,\[
\mathbb{P}\left(\left|X\right|+|Y|\geq r\,\vert\,|Y|\leq M\right)\geq\mathbb{P}\left(\left|X\right|\geq r-M\right).
\]
Therefore, by taking $M$ to be sufficiently large so that $\mathbb{P}(|Y|\leq M)\geq\frac{1}{2}$,
\begin{align*}
\mathbb{P}\left(|X+Y|\geq r\right)&\geq\mathbb{P}\left(|X|+|Y|\geq r\right)\\
&\geq\mathbb{P}\left(\left|X\right|+|Y|\geq r\,\vert\,|Y|\leq M\right)\mathbb{P}\left(\left|Y\right|\leq M\right)\\
&\geq\frac{1}{2}\mathbb{P}\left(\left|X\right|\geq r-M\right)=\Theta(r^{-\alpha}).
\end{align*}
The same process with $X$ and $Y$ reversed implies $\mathbb{P}(|X+Y|\geq r)\geq\Theta(r^{-\beta})$ as well. Both (II) and the claim follow.
\end{proof}

To show Proposition \ref{prop:isotropic-pushforward}, we will require a few extra assumptions to rule out pathological cases. The full content of Proposition \ref{prop:isotropic-pushforward} is contained in the following theorem.

\begin{theorem}
Suppose there exists $\nu > 0$ such that $C:\mathcal{S}^{d-1}\to(0,\infty)$ satisfies $C(v) \coloneqq \lim_{x \to \infty}x^{\nu}\mathbb{P}(|\langle v, X\rangle| > x)$ for all $v \in \mathcal{S}^{d-1}$. If $\nu$ is not an integer and $f$ is a bilipschitz function,
then $f(X)$ is tail-isotropic with tail index $\nu$. 
\end{theorem}
\begin{proof}
Since $x \mapsto \langle v, f(x)\rangle$ is Lipschitz continuous for any $v \in \mathcal{S}^{d-1}$, Theorem \ref{thm:distn_class_closed} implies $\langle v, f(X)\rangle \in \overline{\mathcal{L}_{\nu}^{1}}$. Let $\theta \in (0,\pi/2)$ (say, $\theta = \pi / 4$), and let $S_v = \{x\, : \, \cos^{-1}(\langle x/\|x\|, v\rangle)\leq\theta\}$ for each $v \in \mathcal{S}^{d-1}$. Then
\[
H_v \coloneqq \{x\,:\,\langle v, x \rangle > 1\} \supset \{x\,:\,\|x\| > (1-\cos\theta)^{-1}\} \cap S_v.
\]
From Theorem C.2.1 of \cite{buraczewski2016stochastic}, since $\nu \not \in \mathbb{Z}$, there exists a non-zero measure $\mu$ such that
\[
\mu(E) = \lim_{x \to \infty} \frac{\mathbb{P}(x^{-1}X \in E)}{\mathbb{P}(\|X\| > x)},
\]
for any Borel set $E$. Consequently, $\mu$ is regularly varying, and so 
by the spectral representation of regularly varying random vectors (see p. 281 \cite{buraczewski2016stochastic}), there exists a measure $P$ such that
\[
\lim_{x \to \infty}\frac{\mathbb{P}(\|X\|>tx, X/\|X\| \in E)}{\mathbb{P}(\|X\| > x)} = t^{-\nu} P(E),
\]
for any Borel set $E$ on $\mathcal{S}^{d-1}$ and any $t > 0$. Letting $F_v = \{y / \|y\|\,:\, f(y) \in S_v\} \subset \mathcal{S}^{d-1}$ (noting that $P(F_v) > 0$ by assumption), since $m\|x - y\| \leq \|f(x) - f(y)\| \leq M\|x - y\|$ for all $x,y$,
\begin{align*}
\liminf_{x \to \infty}
\frac{\mathbb{P}(f(X) \in x H_v)}{\mathbb{P}(\|f(X)\| > x)} 
&\geq \liminf_{x \to \infty}\frac{\mathbb{P}(\|f(X)\| > x(1-\cos\theta)^{-1}, f(X) \in S_v)}{\mathbb{P}(\|f(X)\| > x)} \\
&\geq \liminf_{x \to \infty}\frac{\mathbb{P}(\|X\| > x(m(1-\cos\theta))^{-1}, X/\|X\| \in F_v)}{\|X\| > x / M} \\
&\geq P(F_v) \left(\frac{M}{m(1-\cos\theta)}\right)^{-\nu} > 0.
\end{align*}
where $P(F_v) > 0$ follows from the bilipschitz condition for $f$. Therefore, we have shown that $\mathbb{P}(\langle v, f(X)\rangle > x) = \Theta(\mathbb{P}(\|f(X)\| > x))$ for every $v \in \mathcal{S}^{d-1}$.
Since $\mathbb{P}(\|f(X)\| > x)$ obeys a power law with exponent $\nu$ by Corollary \ref{corr:heavy_to_light}, $f(X)$ is tail-isotropic with exponent $\nu$.
\end{proof}


  

\section{Experiments performing VI against a fat-tailed Cauchy target}
\label{sec:cauchy_normal_student}
The motivation for the fat-tailed variational families used in TAF/ATAF
is easily illustrated on a toy example consisting of $X \sim \text{Cauchy}(x_0 = 0, \gamma = 1) \in \cL^1_1$.
As seen in \Cref{fig:cauchy_normal_student}, while ADVI with normalizing flows \citep{kingma2016improved,webb2019improving}
appears to provide a reasonable fit to the bulk of the target distribution (left panel), the improper
imposition of sub-Gaussian tails results in an exponentially bad tail approximation (middle panel).
As a result, samples drawn from the variational approximation fail a Kolmogorov-Smirnov goodness-of-fit
test against the true target distribution much more often (right panel, smaller $p$-values imply more rejections)
than a variational approximation which permits fat-tails. This example is a special case of \Cref{thm:distn_class_closed}.

\begin{figure*}[htbp]
  \centering
  \includegraphics{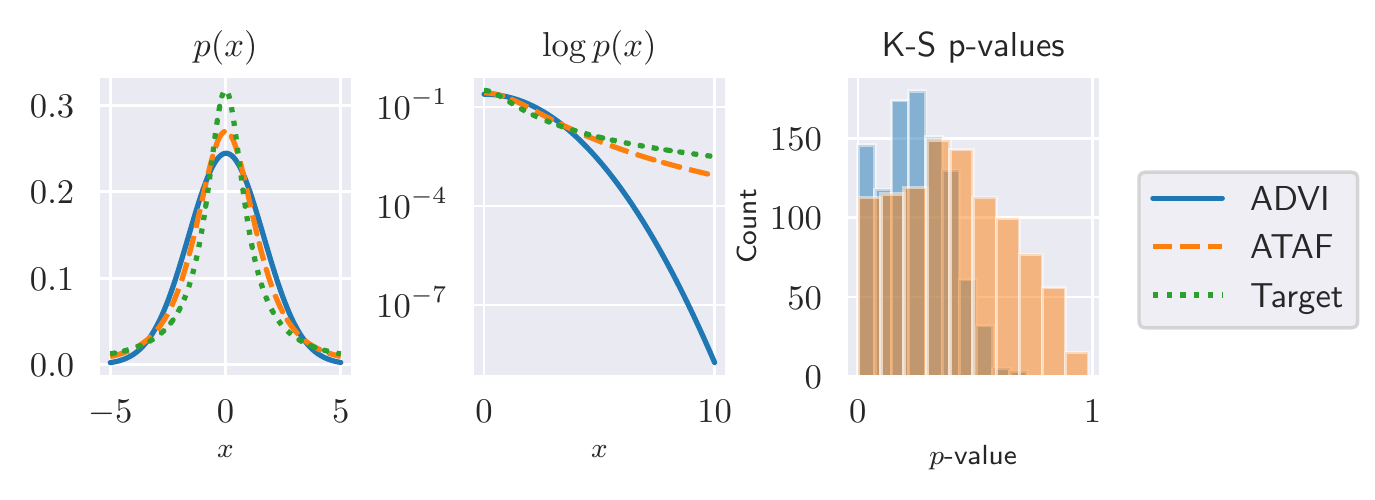}
  \vspace{-6mm}
  \caption{
    When performing FTVI to approximate a $X \sim \text{Cauchy}(x_0 = 0, \gamma = 1)$ target (left panel, green dotted line),
    the use of a sub-Gaussian variational family (ADVI, solid blue line) can incur
    exponentially bad tail approximations (middle panel) compared to
    methods which permit heavier tails (ATAF, green dashed line) and results in
    samples which are more similar to the target (as measured by $1000$ repeats of
    1-sample $N=100$ Kolmogorov-Smirnov $p$-values, right panel).
  }
  \label{fig:cauchy_normal_student}
\end{figure*}

\section{Normal-normal location mixture}
\label{sec:normal-normal-location-mixture}

We consider a Normal-Normal conjugate inference problem where the posterior
is known to be a Normal distribution as well. Here, we aim to show that ATAF
performs no worse than ADVI because $\text{StudentT}(\nu) \to N(0, 1)$ as $\nu \to \infty$.
\Cref{fig:normal_normal} shows the resulting density approximation, which can
be seen to be reasonable for both a Normal base distribution (the ``correct'' one)
and a StudentT base distribution. This suggests that mis-specification (i.e., heavier
tails in the base distribution than the target) may not be too problematic.

\begin{figure}[H]
  \centering
  \includegraphics[width=0.6\textwidth]{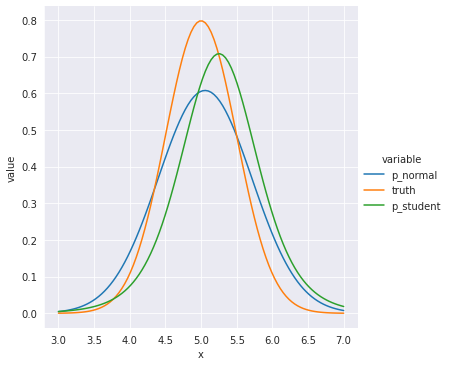}
  \caption{VI against a Normal posterior}
  \label{fig:normal_normal}
\end{figure}

\section{Example of non-existence of tail parameter due to oscillations}
\label{eg:spiral}

Consider $\text{StudentT}(\nu=1) \otimes \text{StudentT}(\nu=2)$ and ``spin'' it
using the radial transformation $(r,\theta) \mapsto (r,r+\theta)$ (\Cref{fig:spiral}). Due to
oscillations, $\alpha_X(v)$ is not well defined for all $v \in \cS^{1}$.

\begin{figure*}[htbp]
    \centering
    \includegraphics[scale=0.8]{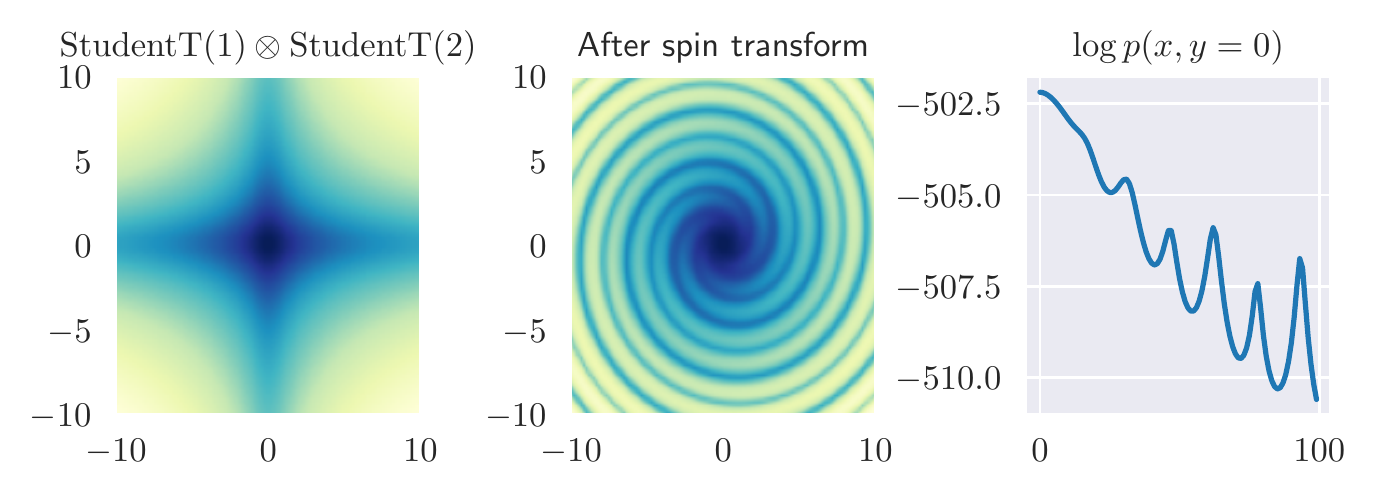}
    \caption{Taking a tail-anisotropic distribution (top) and ``spinning'' it (middle) results in
        one-dimensional projections which oscillate between tail parameters (as seen in
        $\log p(\braket{X,e_0})$ in bottom panel) and cause $\alpha_X(\cdot)$ to be not well defined.
    }
    \label{fig:spiral}
\end{figure*}

\section{Additional details for experiments}
\label{sec:additional-exp-details}

All experiments were performed on an Intel i8700K with 32GB RAM and a NVIDIA GTX 1080
running PyTorch 1.9.0 / Python 3.8.5 / CUDA 11.2 / Ubuntu Linux 20.04 via Windows Subsystem for Linux.
For all flow-transforms $\Phi_{\text{Flow}}$ we used inverse autoregressive flows \citep{kingma2016improved} with a
dense autoregressive conditioner consisting of two layers of either 32 or 256 hidden units depending on problem (see code for details) and
ELU activation functions.
As described in \cite{jaini2020tails}, TAF is trained by including $\nu$ within the Adam optimizer alongside other flow parameters. For ATAF, we include all $\nu_i$ within the optimizer.
Models were trained using the Adam optimizer with $10^{-3}$ learning rate
for 10000 iterations, which we found empirically in all our experiments to result in negligible change in ELBO
at the end of training.

For \cref{tab:diamonds} and \cref{tab:eight_schools}, the flow transform $\Phi_{\text{Flow}}$ used for ADVI, TAF, and ATAF
are comprised of two hidden layers of 32 units each. NUTS uses no such flow transform. Variational parameters for each normalizing flow were initialized
using \texttt{torch}'s default Kaiming initialization \citep{he2015delving} Additionally, the tail parameters $\nu_i$
used in ATAF were initialized to all be equal to the tail parameters learned from training TAF. We empirically observed
this resulted in more stable results (less variation in ELBO / $\log p(y)$ across trials), which may be due to
the absence of outliers when using a Gaussian base distribution resulting in more stable ELBO gradients. This suggests
other techniques for handling outliers such as winsorization may also be helpful, and we leave further investigation
for future work.

For \cref{fig:blr-anisotropic}, the closed-form posterior was computed over a finite element grid to produce
the ``Target'' row. A similar progressive training scheme used for \cref{tab:diamonds} was also used here, with
the TAF flow transform $\Phi_{\text{Flow}}$ was initialized from the result of ADVI and ATAF additionally initialized
all tail parameters $\nu_i$ based on the final shared tail parameter obtained from TAF training. Tails are computed
along the $\beta = 1$ or $\sigma = 1$ axes because the posterior is identically zero for $\sigma = 0$ hence it reveals
no information about the tails.

\end{document}